\documentclass[twoside]{article}

\usepackage[accepted]{aistats2021}

\usepackage{amsmath,amsfonts,bm}

\def\eqref#1{equation~\ref{#1}}

\def\1{\bm{1}}

\def\vu{{\bm{u}}}
\def\vv{{\bm{v}}}

\def\vx{{\bm{x}}}
\def\vy{{\bm{y}}}
\def\vz{{\bm{z}}}

\DeclareMathAlphabet{\mathsfit}{\encodingdefault}{\sfdefault}{m}{sl}
\SetMathAlphabet{\mathsfit}{bold}{\encodingdefault}{\sfdefault}{bx}{n}

\usepackage{hyperref}
\usepackage{url}
\usepackage{fontawesome}
\usepackage{amsthm}
\usepackage{enumitem}
\usepackage{adjustbox}
\usepackage{graphicx}

\theoremstyle{definition}
\newtheorem{definition}{Definition}[section]

\theoremstyle{lemma}
\newtheorem{lemma}{Lemma}[section]

\theoremstyle{proposition}
\newtheorem{proposition}{Proposition}[section]

\theoremstyle{theorem}
\newtheorem{theorem}{Theorem}[section]

\theoremstyle{corollary}
\newtheorem{corollary}{Corollary}[section]

\begin{document}

%

%
\runningauthor{KXY Technologies, Inc.}

\twocolumn[

\aistatstitle{Inductive Mutual Information Estimation: A Convex Maximum-Entropy Copula Approach}

\aistatsauthor{Yves-Laurent Kom Samo}

\aistatsaddress{KXY Technologies, Inc. \\
\faEnvelope \, \texttt{yl@kxy.ai} \faMedium \,  \texttt{@Dr\_YLKS}  \faTwitter  \, \texttt{@Dr\_YLKS} \\
San Jose, California, USA} ]

\begin{abstract}
We propose a novel estimator of the mutual information between two ordinal vectors $\displaystyle \vx$ and $\displaystyle \vy$. Our approach is inductive (as opposed to deductive) in that it depends on the data generating distribution solely through some nonparametric properties revealing associations in the data, and does not require having enough data to fully characterize the true joint distributions $P_{\displaystyle \vx, \displaystyle \vy}$. Specifically, our approach consists of (i) noting that $I\left(\displaystyle \vy; \displaystyle \vx\right) = I\left(\displaystyle \vu_y; \displaystyle \vu_x\right)$ where $\displaystyle \vu_y$ and $\displaystyle \vu_x$ are the \emph{copula-uniform dual representations} of $\displaystyle \vy$ and $\displaystyle \vx$ (i.e. their images under the probability integral transform), and (ii) estimating the copula entropies $h\left(\displaystyle \vu_y\right)$, $h\left(\displaystyle \vu_x\right)$ and $h\left(\displaystyle \vu_y, \displaystyle \vu_x\right)$ by solving a maximum-entropy problem over the space of copula densities under a constraint of the type $\bm{\alpha}_m = E\left[\phi_m(\displaystyle \vu_y, \displaystyle \vu_x)\right]$. We prove that, so long as the constraint is feasible, this problem admits a unique solution, it is in the exponential family, and it can be learned by solving a convex optimization problem. The resulting estimator, which we denote MIND, is marginal-invariant, always non-negative, unbounded for any sample size $n$, consistent, has MSE rate $O(1/n)$, and is more data-efficient than competing approaches.
\end{abstract}

\section{Introduction}
Mutual information plays a key role in statistical learning. It is directly related to the highest $R^2$, the highest true log-likelihood per observation, the lowest root mean square error and the highest classification accuracy that can be achieved by using explanatory variables $\displaystyle \vx$ to predict categorical or continuous output(s) $\displaystyle \vy$. It also plays an important role in representation learning (\cite{brown1992class,bell1995information,tishby2000information,tishby2015deep, chen2016infogan,higgins2018towards}) and reinforcement learning (\cite{pathak2017curiosity, oord2018representation}). 

Virtually every mutual information estimator in the litterature implicitly assumes that we have a number $n$ of i.i.d. samples $\left(\displaystyle \vx_1, \displaystyle \vy_1 \right), \dots, \left( \displaystyle \vx_n, \displaystyle \vy_n \right)$ that is large enough to characterize the underlying distribution $P_{\displaystyle \vx, \displaystyle \vy}$. We will refer to this scenario as the \emph{deductive approach}. Examples include quantizing (\cite{paninski2003estimation}) or hashing (\cite{noshad2019scalable}) $\displaystyle \vx$ and $\displaystyle \vy$ and computing the mutual information between the resulting discrete distributions using sample frequencies. Other approaches approximate the pdfs using kernel density estimators (\cite{moon1995estimation, kwak2002input}), using local geometric properties based on k nearest neighbors (\cite{kraskov2004estimating, gao2015efficient}) and using Edgeworth approximation (\cite{hulle2005edgeworth}). Another perspective has been to learn lower bounds based on variational characterizations of the mutual information (\cite{nguyen2010estimating, belghazi2018mutual}) using M-estimators.

The \emph{deductive approach} is fraught with limitations. When relying on discrete approximations, the mutual information can never be greater than the mutual information between fully dependent uniform distributions, namely $\log n$ in the case of quantization (\cite{paninski2003estimation}) and $\log F$ in the case of hashing (\cite{noshad2019scalable}), where $F < n$ is the number of distinct hashes and $n$ the sample size. \cite{mcallester2020formal} extended this result empirically to various continuous mutual information estimators relying on the deductive approach, including variational estimators (\cite{nguyen2010estimating, belghazi2018mutual}). Their work touches on the core issue: if we require properly characterizing the joint pdf nonparametrically from $n$ i.i.d. samples in order to estimate a mutual information, then $n$ ought to be large, otherwise we will not see enough tail events, and we will fail to account for tail dependency. Unfortunately, the alternative proposed by \cite{mcallester2020formal}, namely approximating the density $p\left(\displaystyle \vx \right)$ and  $p\left(\displaystyle \vx \vert \displaystyle \vy \right)$ parametrically using deep neural networks is data inefficient and tends to overshoot on small sample sizes. Along the same line, \cite{song2019understanding} and \cite{poole2019variational} reported that the variance of variational estimators MINE (\cite{belghazi2018mutual}) and NWJ (\cite{nguyen2010estimating}) could be very large, and even grow exponentially with the true mutual information (\cite{song2019understanding}), due to the need for a large number $n$ of samples to accurately estimate an expectation of the form $E\left[ e^{T\left(\displaystyle \vx,  \displaystyle \vy \right)} \right]$ under $P_{\displaystyle \vx} \otimes P_{\displaystyle \vy}$.

Another major limitation of the deductive approach is the unnecessary need to accurately model the marginals of $P_{\displaystyle \vx, \displaystyle \vy}$, directly or implicitly, as a pre-requisite for estimating the mutual information, which could be data and compute intensive, even though the mutual information does not depend on marginal distributions.

The \emph{inductive approach} we introduce in this paper is structured in two stages. First, we measure a few nonparametric properties of the data generating distribution that serve as marginal-invariant proxies revealing associations between coordinates of $\displaystyle \vx$ and/or $\displaystyle \vy$. Then we estimate the mutual information in the spirit of the maximum-entropy principle (\cite{jaynes1957informationi, jaynes1957informationii}), by being consistent with all observed properties, while remaining as uninformative as possible about any property we haven't observed. Intuitively, we would expect that the more expressive the properties we measure get, the closer we should get to the true mutual information. Indeed, we propose a family of nonparametric properties that give rise to a consistent estimator of the true mutual information.

The rest of the paper is structured as follows. In Section \ref{sct:background} we recall some results relating copulas and mutual information. In Section \ref{sct:MIND} we further motivate our inductive approach and we present core theoretical results pertaining to maximum-entropy inference of copulas. Our theoretical contribution builds on the study of the $I$-divergence geometry of probability distributions developed by \cite{csiszar1975divergence}, which generalizes the minimum discrimination information theorem of \cite{kullback1966note}. In Section \ref{sct:estimation} we propose a convex pro
gram for solving maximum-entropy copula problems under linear constraints, and we discuss practical considerations. Finally, in Section \ref{sct:application} we illustrate that our estimator outperforms the state of the art on large and small mutual information problems on synthetic data, we illustrate that our work can be used to mitigate mode collapse in GANs, and we apply our approach to the estimation of the highest performance achievable in a Kaggle competition.

\section{Background}
\label{sct:background}
We begin by recalling that the \emph{mutual information} between two random vectors $\displaystyle \vx$ and $\displaystyle \vy$ is defined as 
\begin{align*}
I( \displaystyle \vy; \displaystyle \vx) :&= \int_{\mathcal{X} \times \mathcal{Y}} \log  \frac{d P_{\displaystyle \vx, \displaystyle \vy}}{dP_{\displaystyle \vx} \otimes P_{\displaystyle \vy}}  dP_{\displaystyle \vx, \displaystyle \vy}
\end{align*}
where $P_{\displaystyle \vx, \displaystyle \vy}$ (resp. $P_{\displaystyle \vx}$, $P_{\displaystyle \vy}$) is the (joint) probability measure of $(\displaystyle \vx, \displaystyle \vy)$ (resp. $\displaystyle \vx$, $\displaystyle \vy$), and $\frac{d P_{\displaystyle \vx, \displaystyle \vy}}{dP_{\displaystyle \vx} \otimes P_{\displaystyle \vy}}$ is the Radon-Nikodym derivative of the joint probability measure with respect to the product measure of $P_{\displaystyle \vx}$ and $P_{\displaystyle \vy}$. Friendlier expressions depending on whether the random vectors have continuous and/or categorical coordinates are provided in Table \ref{tab:mutual_information} in the appendix.

While the mutual information is the canonical approach for \emph{quantifying} associations betweeen random variables, copulas are the canonical tool for \emph{modeling} associations between random variables. We recall some basic definitions and properties, and link the two notions.
\begin{definition}
A \textbf{copula distribution} is any probability distribution supported on $[0, 1]^d$ whose marginals are uniform. A \textbf{copula} (resp. \textbf{copula density}) is any function that is the cdf (resp. pdf) of a copula distribution.
\end{definition}
The following theorem shows that every distribution with pdf is uniquely associated to a copula density that fully captures its dependence structure, independently from marginals. 
\begin{theorem}{(Sklar's Theorem)}\label{theo:sklar} Any pdf $f: \mathcal{Z} \subset \mathbb{R}^d \to \mathbb{R}^+$ whose marginal pdfs are $f_1(z_1), \dots, f_d(z_d)$ with associated cdfs $F_1(z_1), \dots, F_d(z_d)$ can be uniquely decomposed as 
\begin{align*}
f\left(z_1, \dots, z_d \right) = c\left(F_1\left( z_1 \right), \dots, F_d\left( z_d \right) \right) \prod_{i=1}^d f_i\left( z_i \right),
\end{align*}
where $c: [0, 1]^d \to \mathbb{R}^+$ is a \emph{copula density}. We refer to $\displaystyle \vu_z := \left(F_1(z_1), \dots, F_d(z_d) \right)$ as the \textbf{copula-uniform dual representation} of $\displaystyle \vz := \left( z_1, \dots, z_d \right)$, and to $\displaystyle \vz$ as a \textbf{primal representation} of $\displaystyle \vu_z$.
\end{theorem}
Interestingly, the entropy of a copula-uniform dual representation is invariant  by continuous 1-to-1 primal feature transformations.
\begin{proposition}
\label{prop:inv}
If $\displaystyle \vu_z$ is the copula-uniform dual representation of $\displaystyle \vz := \left(z_1, \dots, z_d\right)$ and $\displaystyle \vu_{g(z)}$ is the copula-uniform dual representation of $g(\displaystyle \vz) := \left(g_1(z_1), \dots, g_d(z_d)\right)$ where the functions $g_1, \dots, g_d$ are continuous 1-to-1 functions, then $$h\left(\displaystyle \vu_z\right) = h\left( \displaystyle \vu_{g(\bm{z})} \right).$$
\end{proposition}
See Appendix \ref{proof:prop:inv} for the proof.

The following entropy decomposition is a direct consequence of Sklar's theorem.
\begin{proposition}
\label{prop:entropy_decomp}
If $\displaystyle \vu_z$ is the copula-uniform dual representation of $\displaystyle \vz := \left(z_1, \dots, z_d\right) \in \mathcal{Z} \subset \mathbb{R}^d$ and $\displaystyle \vz$ admits a pdf, then the differential entropy of $\displaystyle \vz$ can be decomposed as
\begin{align}
\label{eq:ent_decomp}
h\left(\displaystyle \vz\right) = h\left(\displaystyle \vu_z\right) + \sum_{i=1}^d h\left( z_i \right),
\end{align}
so long as all marginal entropies exist. $h\left(\displaystyle \vu_z\right)$ is the entropy of the associated copula distribution, and we refer to it as the  \textbf{copula entropy} of $\displaystyle \vz$.
\end{proposition}
This entropy decomposition implies that the mutual information between two continuous random vectors is the same as that of their copula-uniform dual representations:
\begin{align*}
I\left(\displaystyle \vy; \displaystyle \vx \right) &= h\left(\displaystyle \vy\right) + h\left(\displaystyle \vx \right) - h\left(\displaystyle \vy, \displaystyle \vx \right) \\
&= h\left(\displaystyle \vu_y\right) + h\left(\displaystyle \vu_x \right) - h\left(\displaystyle \vu_y, \displaystyle \vu_x \right) \\
&=  I\left(\displaystyle \vu_y; \displaystyle \vu_x \right).
\end{align*}
The identity $I\left(\displaystyle \vy; \displaystyle \vx \right)=I\left(\displaystyle \vu_y; \displaystyle \vu_x \right)$ extends to all ordinal random vectors by noting that copula-uniform dual representations are well defined for ordinal random vectors, are in a 1-to-1 relationship with their primal representations, and that the mutual information is invariant by 1-to-1 maps. In general, we may use Table \ref{tab:mutual_information} to conclude that estimating any mutual information boils down to estimating copula entropies, and possibly one-dimensional primal entropies (when the problem involves categorical and non-ordinal coordinates that we choose not to ordinally encode).

Another important property of the mutual information we will rely on is that it is stable by addition of redundant information.
\begin{proposition}
\label{prop:redund}
Let $\displaystyle \vy \in \mathcal{Y}$ and $\displaystyle \vx \in \mathcal{X}$ be two random vectors, and $f$ a function defined on $\mathcal{X}$. Then we have:
\begin{align*}
I\left(\displaystyle \vy; \displaystyle \vx, f\left(\displaystyle \vx\right) \right) = I\left(\displaystyle \vy; \displaystyle \vx \right) + \underbrace{I\left(\displaystyle \vy; f\left(\displaystyle \vx \right) \vert \displaystyle \vx\right)}_{=0} = I\left(\displaystyle \vy; \displaystyle \vx \right).
\end{align*}
\end{proposition}

Going forward, and without loss of generality, we will focus on the estimation of the copula entropy of a continuous random vector $h\left( \displaystyle \vu_z \right)$, with the understanding that the mutual information is recovered as $I\left(\displaystyle \vy; \displaystyle \vx \right)=h\left(\displaystyle \vu_y\right) + h\left(\displaystyle \vu_x \right) - h\left(\displaystyle \vu_y, \displaystyle \vu_x \right)$.

\section{Inductive Mutual Information Estimation}
\label{sct:MIND}
We consider estimating the copula entropy $h\left(\displaystyle \vu_z\right)$ of a continuous random vector $\displaystyle \vz \in \mathcal{Z} \subset \mathbb{R}^d$, where $\displaystyle \vu_z$ is the copula-uniform dual representation of $\displaystyle \vz$, which we assume admits a pdf.

\subsection{Motivation and Roadmap}
The \emph{deductive approach} to learning the copula entropy $h\left(\displaystyle \vu_z\right)$ requires assuming that we have gathered enough samples to fully characterize the pdf $p(\displaystyle \vu_z)$, estimating the pdf, and then estimating the copula entropy as the entropy of the estimated pdf. This is both data and compute inefficient. If we partition $[0, 1]^d$ into small hypercubes of side length $\delta$, then we need to observe at least one sample per hypercube for a small enough $\delta$ to properly characterize the pdf nonparametrically. This requires $n \propto \delta^{-d}$ samples. How small $\delta$ needs to be depends on how quickly the true pdf varies on $[0, 1]^d$. Regardless, $n$ would grow exponentially with the input dimension, and so would the associated compute requirement. 

This inefficiency can be alleviated by assuming that the pdf belongs to a specific parametric family, at the expense of model mispecification. If the parametric family has sufficient statistics $T\left(\displaystyle \vu_1, \dots, \displaystyle \vu_n \right)$ for  $n$ i.i.d. observations, then pdfs in the family should be maximum-entropy among all pdfs with the same statistics. If this is not the case, then the parametric family would be violating Occam's razor as the learned pdf would be encoding more structure than evidenced by the data, and the estimated entropy would overshoot. We also note that the pdf $p(\displaystyle \vu_z)$ should have uniform marginals, which makes finding an appropriate parametric family even more difficult.

In the absence of any empirical evidence, Occam's razor suggests that the most appropriate distribution for $\displaystyle \vu_z$ is the uniform distribution on $[0, 1]^d$ as it is the least informative (or maximum-entropy) of all distributions supported on $[0, 1]^d$. Instead of choosing a rigid parametric family and inheriting its sufficient statistics, we could construct more expressive copula densities by first choosing how to reveal departure from the standard uniform distribution from the data, and then finding the least informative copula density among all copula densities satisfying the observed constraints.

The \emph{inductive approach} we propose consists of revealing the dependence structure in $\displaystyle \vu_z$ by estimating an expectation of the form $\bm{\alpha}_m = E_{P_{\displaystyle \vu_z}} \left[ \phi_m\left( \displaystyle \vu_z \right) \right]$, for a vector-valued statistics function $\phi_m: [0, 1]^d \to \mathbb{R}^q$, and estimating $h(\displaystyle \vu_z)$ as the highest copula entropy among all copulas satisfying the constraint $\bm{\alpha}_m = E_{P}\left[ \phi_m\left( \displaystyle \vu_z \right) \right]$. When $\phi_m$ is given and $E_{P_{\displaystyle \vu_z}} \left[ \phi_m\left( \displaystyle \vu_z \right) \right]$ is all the data scientist can reliably observe about the structure of the data (e.g. we are only given pairwise Spearman rank correlations), our estimator is the only estimator consistent with Occam's razor. Crucially, if we may choose $\phi_m$, then we may approximate the copula entropy with arbitrary precision. Specifically, we show that, so long as $\left(\phi_m\right)_m$ are universal approximators of continuous functions on $[0, 1]^d$, the solution to this maximum-entropy problem is a consistent estimator of the true copula entropy $h\left(\displaystyle \vu_z\right)$ (see Theorem \ref{theo:fund}). Equally important is Corollary \ref{cor:ind_wins} that states that we may perfectly recover the true mutual information using a finite dimensional statistics function that is not expressive enough to fully characterize the true copula distribution $P_{\displaystyle \vu_{x, y}}$.

We note that $\bm{\alpha}_m$ can be efficiently estimated from $n$ i.i.d. primal samples $\displaystyle \vz_1, \dots, \displaystyle \vz_n$ as 
\begin{align}
\label{eq:estm}
\bm{\hat{\alpha}}_{m, n} = \frac{1}{n} \sum_{i=1}^n \phi_m\left( \frac{\text{rg}\left(\displaystyle \vz_i \right)}{n+1} \right),
\end{align}
where $\text{rg}\left(\displaystyle \vz_i \right)$ is the vector of coordinatewise ranks of $\displaystyle \vz_i $ among $\displaystyle \vz_1, \dots, \displaystyle \vz_n$. It follows from the weak convergence of the empirical copula process to the true copula that $\bm{\hat{\alpha}}_{m, n}$ is a consistent and asymptotically normal estimator of $\bm{\alpha}_m$ (\cite{ruschendorf1976asymptotic}). Thus, our \emph{inductive approach} to mutual information estimation truly does not require learning marginal distributions. We show that, so long as the constraint $E_{P}\left[ \phi_m\left( \displaystyle \vu_z \right) \right] = \bm{\hat{\alpha}}_{m, n}$ is feasible, the associated maximum-entropy problem admits a unique solution, and it is a consistent estimator\footnote{Consistency here is jointly in $n$ and $m$.} of $h\left(\displaystyle \vu_z\right)$ when $\left(\phi_m\right)_m$ are universal approximators of continuous functions on $[0, 1]^d$. Finally, we introduce a convex optimization problem whose minimizer is the maximizer of our maximum-entropy problem. 

\subsection{Maximum-Entropy Copulas}
Let $\mathcal{C}_d$ be the space of all $d$-dimensional copula distributions with pdf, $h\left(P\right)$ the differential entropy of the probability distribution $P$, and $\phi_{m,d}: [0, 1]^d \to \mathbb{R}^{q(d)}$ a vector-valued function whose coordinate functions are not linearly dependent. 

We define the following properties of $\phi_{m,d}$: (P1) the first coordinate of $\phi_{m, d}$ is the constant $1$, (P2) each coordinate of $E_P \left[ \phi_{m, d} \left( \displaystyle \vu \right) \right]$ captures a way in which $P$ departs from the uniform distribution, (P3) the family $\left( \phi_{m,d} \right)_m$ is a universal approximator of continuous functions defined on $[0, 1]^d$ for any $d$, and (P4) $q(l) + q(k) \leq q(l+k)$ for every $l, k>0$ and all coordinates of $\phi_{m,l}$ and $\phi_{m,k}$ are also included in $\phi_{m,l+k}$. Going forward, we will use $\phi_m$ and $q$ in-lieu-of $\phi_{m,d}$ and $q(d)$ for ease of notation when the input dimension is unambiguous.

We consider the following optimization problem:
\begin{align}
\label{MIND}
\begin{cases}
\underset{P \in \mathcal{C}_d}{\max} ~~~ h\left( P \right) \tag{MIND}  \\
\text{s.t.} ~ E_P\left[ \phi_m \left(\displaystyle \vu \right)  \right] = \bm{\alpha}_m \nonumber
\end{cases}.
\end{align}
The theorem below, which we prove in Appendix \ref{proof:theo:MIND}, states that the solution of (\ref{MIND}) is an exponential family distribution with sufficient statistics $\phi_m$, and with base measure the product of $d$ measures on $[0, 1]$ that are absolutely continuous with respect to the standard uniform on $[0, 1]$.
\begin{theorem}
\label{theo:MIND}
Let $\phi_m$ satisfy (P1). If there is any copula distribution $P$ with finite differential entropy and satisfying $E_P\left[ \phi_m \left(\displaystyle \vu \right)  \right] = \bm{\alpha}_m$, then the maximum-entropy problem (\ref{MIND}) admits a unique solution, and the maximizer $P_{\text{M}}$ is the only copula distribution whose density takes the form
\begin{align}
p_{\text{M}} \left(\displaystyle \vu; \phi_m,  \bm{\alpha}_m \right) = e^{\bm{\theta}^T \phi_m \left(\displaystyle \vu \right)} \prod_{i=1}^d f_i\left(u_i\right),
\end{align}
and that satisfies the constraint $E_{P_{\text{M}}}\left[ \phi_m \left(\displaystyle \vu \right)  \right] = \bm{\alpha}_m$ for some constant $\bm{\theta}$, and $d$ non-negative univariate functions $f_i$, $\log f_i \in L_1\left([0, 1]\right)$. Moreover, for any copula distribution $P \in \mathcal{C}_d$ satisfying $E_P\left[ \phi_m \left(\displaystyle \vu \right)  \right] = \bm{\alpha}_m$, 
\begin{align}
h\left( P_{\text{M}} \right) - h(P) = KL\left( P \vert \vert P_{\text{M}} \right).
\end{align}
\end{theorem}
The practical challenge with applying Theorem \ref{theo:MIND} is the need to learn the free functions $f_1, \dots, f_d$. These functions ensure that the maximizer is a copula distribution---i.e. has uniform marginals. We now consider relaxing this requirement and solving the maximum-entropy problem over the space $\mathcal{D}_d \supset \mathcal{C}_d$ of all continuous probability distributions supported on $[0, 1]^d$:
\begin{align}
\label{A-MIND}
\begin{cases}
\underset{P \in \mathcal{D}_d}{\max} ~~~ h\left( P \right) \tag{A-MIND}  \\
\text{s.t.} ~ E_P\left[ \phi_m \left(\displaystyle \vu \right)  \right] = \bm{\alpha}_m \nonumber
\end{cases}.
\end{align}
As we later show in Theorem \ref{theo:fund}, (P3) guarantees that, despite this relaxation, the solution to (\ref{A-MIND}) converges to the true copula entropy as $m$ goes to infinity. For a given $m$, to control how close to uniform the maximizer's marginals are, we use the moment characterization of the standard uniform $\forall j, ~E(u^j ) = 1/(1+j)$, and we match the first $k$ moments of marginals of candidate distributions to those of the standard uniform. To do so, we introduce the property (P5): $\phi_m$ has the form $$\phi_m^k( \displaystyle \vu) = \left(1, u_1, \dots, u_1^k, \dots, u_d, \dots, u_d^k, \psi_m\left( \vu \right) \right),$$ with associated $$\bm{\alpha}_m^k = \left(1, 1/2, \dots, 1/(1+k), \dots, 1/2, \dots,  1/(1+k), \bm{\beta}_m \right).$$ The solution of (\ref{A-MIND}) is provided by the following theorem, which we prove in Appendix \ref{proof:theo:a-MIND}.
\begin{theorem}
\label{theo:a-MIND}
Let $\phi_m^k$ satisfy (P5). If there is any distribution $P$ supported on $[0, 1]^d$, with finite differential entropy and satisfying $E_P\left[ \phi_m^k \left(\displaystyle \vu \right)  \right] = \bm{\alpha}_m^k$, then the maximum-entropy problem (\ref{A-MIND}) admits a unique solution of the form
\begin{align}
h_{\text{AM}} \left(\bm{u}; \phi_m^k,  \bm{\beta}_m \right) = -\bm{\theta}^T \bm{\alpha}_m^k,
\end{align}
and the maximizer $P_{\text{AM}}$ is the only distribution supported on $[0, 1]^d$, whose pdf takes the form 
\begin{align}
p_{\text{AM}} \left(\displaystyle \vu; \phi_m^k,  \bm{\beta}_m \right) = e^{\bm{\theta}^T \phi_m^k \left(\displaystyle \vu \right)}
\end{align}
for some constant $\bm{\theta}$, and that satisfies the constraint $E_{P_{\text{AM}}}\left[ \phi_m^k \left(\displaystyle \vu \right)  \right] = \bm{\alpha}_m^k$. Moreover, for any distribution $P \in \mathcal{D}_d$ satisfying $E_P\left[ \phi_m^k \left(\displaystyle \vu \right)  \right] = \bm{\alpha}_m^k$, 
\begin{align} h\left( P_{\text{AM}} \right) - h(P) = KL\left( P \vert \vert P_{\text{AM}} \right).
\end{align}
\end{theorem}
\begin{corollary}
\label{cor:ind_wins}
Let $\displaystyle \vx \in \mathcal{X}$ and $\displaystyle \vy \in \mathcal{Y}$ be two continuous random variables with mutual information $I\left( \displaystyle \vy; \displaystyle \vx \right)$ and true individual and joint copula distributions $P_{\displaystyle \vu_x}$, $P_{\displaystyle \vu_y}$,and $P_{\displaystyle \vu_{x,y}}$. If $P_\text{AM}\left( \displaystyle \vu_x \right)$, $P_\text{AM}\left( \displaystyle \vu_y \right)$ and $P_\text{AM}\left( \displaystyle \vu_x, \displaystyle \vu_y \right)$ are the solutions to three (\ref{A-MIND}) problems whose constraints are satisfied by the true copula distributions, and 
\begin{align*}
I_{\text{AM}}\left( \displaystyle \vy; \displaystyle \vx \right) :&= h\left( P_\text{AM}\left( \displaystyle \vu_x \right) \right) +  h\left( P_\text{AM}\left( \displaystyle \vu_y \right) \right) \\
&- h\left(P_\text{AM}\left( \displaystyle \vu_x, \displaystyle \vu_y \right) \right)
\end{align*} is the associated mutual information estimator, then
\begin{align}
\label{eq:fund_ind}
&I_{\text{AM}}\left( \displaystyle \vy; \displaystyle \vx \right) - I\left( \displaystyle \vy; \displaystyle \vx \right) =  KL\left[ P_{\displaystyle \vu_x} \vert \vert P_\text{AM}\left( \displaystyle \vu_x \right) \right] \\
&~ + KL\left[ P_{\displaystyle \vu_y} \vert \vert  P_\text{AM}\left( \displaystyle \vu_y \right) \right] - KL\left[ P_{\displaystyle \vu_{x,y}} \vert \vert P_\text{AM}\left( \displaystyle \vu_{x,y} \right) \right]. \nonumber 
\end{align}
\end{corollary}
Equation (\ref{eq:fund_ind}) in Corollary \ref{cor:ind_wins} lays out the theoretical ground for favoring our \emph{inductive approach} over the traditional \emph{deductive approach}. Indeed, it shows that it is not necessary to accurately learn the true data generating distributions $P_{\displaystyle \vx}$, $P_{\displaystyle \vy}$, and $P_{\displaystyle \vx,\displaystyle \vy}$ or their copulas in order to accurately learn the mutual information $I\left( \displaystyle \vy; \displaystyle \vx \right)$. The error made by (\ref{A-MIND}) in estimating the joint copula entropy $h\left( \displaystyle \vu_y, \displaystyle \vu_x \right)$ can offset the errors made estimating the individual copula entropies $h\left( \displaystyle \vu_y \right)$ and $h\left( \displaystyle \vu_x \right)$, so that we may perfectly estimate the mutual information with a finite dimensional statistics function $\phi_m$, without accurately learning the copula distributions.

\subsection{Iterative MIND}
\label{sct:sparse}
When either (\ref{MIND}) or (\ref{A-MIND}) are sparse in the sense that each coordinate of $\phi_m^k$ is a function of some but not all coordinates of $\bm{u}$, they can be broken down into smaller, cheaper, more robust and cacheable  problems. 
\begin{theorem}
\label{theo:sparse_1}
Let us assume that $\psi_m\left(\displaystyle \vu \right)$ takes the form $\psi_m\left(\displaystyle \vu \right) = \left(\eta_1\left(w, \displaystyle \vv_1\right), \dots,  \eta_q\left(w, \displaystyle \vv_q\right) \right)$ where $w$ is a coordinate of $\displaystyle \vu$ and the vectors $\displaystyle \vv_1, \dots, \displaystyle \vv_q$ are made of coordinates of $\displaystyle \vu$ but share no common coordinate and do not include $w$. If the constraint $E_P\left[ \psi_m \left(\displaystyle \vu \right)  \right] = \bm{\beta}_m := \left(\bm{\beta}_1, \dots, \bm{\beta_q} \right)$ is feasible, then we have 
\begin{align}
p_{\text{M}} \left(\displaystyle \vu; \phi_m^k,  \bm{\alpha}_m^k \right) &= \prod_{i=1}^q p_{\text{M}} \left(w, \displaystyle \vv_i; \eta_i,  \bm{\beta}_i \right), \\
h_{\text{M}} \left(\displaystyle \vu; \phi_m^k,  \bm{\alpha}_m^k \right) &= \sum_{i=1}^q h_{\text{M}} \left(w, \displaystyle \vv_i; \eta_i,  \bm{\beta}_i \right).
\end{align}
\end{theorem}
The proof is provided in Appendix \ref{proof:theo:sparse_1}. Theorem \ref{theo:sparse_1} is useful in single-output problems when blocks of explanatory variables are known or assumed to be independent conditional on the output. 

More generally, when each coordinate of $\phi_m^k$ depends on some but not all coordinates of $\displaystyle \vu$, the problem (\ref{A-MIND}) can be broken down into smaller problems that can be solved iteratively.

\begin{theorem}
\label{theo:sparse_2}
Let $\displaystyle \vu = \left(\displaystyle \vv_1, \dots, \displaystyle \vv_q \right)$ and let us assume that $\psi_m\left(\displaystyle \vu \right)$ takes the form $\psi_m\left(\displaystyle \vu \right) = \left(\eta_1\left( \displaystyle \vv_1\right), \dots,  \eta_q\left(\displaystyle \vv_q\right), \gamma\left(\displaystyle \vu\right) \right)$ where all $\eta_i$ and $\gamma$ satisfy (P1). If the constraint $E_P\left[ \psi_m \left(\displaystyle \vu \right)  \right] = \bm{\beta}_m := \left(\bm{\beta}_1, \dots, \bm{\beta_q}, \bar{\bm{\beta}} \right)$ is feasible, then the maximizer of (\ref{A-MIND}) has density of the form
\begin{align}
p_{\text{AM}} \left(\displaystyle \vu; \phi_m^k,  \bm{\beta}_m \right) = e^{\bm{\theta}^T \phi_m^k\left(\displaystyle \vu\right)}\prod_{i=1}^q p_{\text{AM}} \left(\displaystyle \vv_i; \eta_i,  \bm{\beta}_i \right)
\end{align}
where $\bm{\theta}$ is the only constant such that $E_{P_{\text{AM}}}\left[ \phi_m^k \left(\displaystyle \vu \right)  \right] = \bm{\alpha}_m^k$. 

Moreover, the solution reads
\begin{align*}
& h_{\text{AM}} \left(\displaystyle \vu; \phi_m^k,  \bm{\beta}_m \right)  = -\bm{\theta}^T\bm{\alpha}_m^k + \sum_{i=1}^q h_{\text{AM}} \left(\displaystyle \vv_i; \eta_i,  \bm{\beta}_i \right)\\
&~~~~-\sum_{i=1}^q \text{KL}\left[ p_{\text{AM}} \left(\displaystyle \vv_i; \phi_m^k,  \bm{\beta}_m \right) \vert \vert p_{\text{AM}} \left(\displaystyle \vv_i; \eta_i,  \bm{\beta}_i \right)  \right],
\end{align*}
where $p_{\text{AM}} \left(\displaystyle \vv_i; \phi_m^k,  \bm{\beta}_m \right)$ is the marginal of the maximizer of the full (\ref{A-MIND}) problem and $p_{\text{AM}} \left(\displaystyle \vv_i; \eta_i,  \bm{\beta}_i \right)$ the maximizer of the (\ref{A-MIND}) problem pertaining to $\displaystyle \vv_i$. 
Furthermore, $$-\bm{\theta}^T\bm{\alpha}_m^k \leq \sum_{i=1}^q \text{KL}\left[ p_{\text{AM}} \left(\displaystyle \vv_i; \phi_m^k,  \bm{\beta}_m \right) \vert \vert p_{\text{AM}} \left(\displaystyle \vv_i; \eta_i,  \bm{\beta}_i \right)  \right],$$ and the equality holds if and only if $\bm{\theta}=0$, condition satisfied if $\forall \displaystyle \vu, ~\gamma\left(\displaystyle \vu\right)=1$.
\end{theorem}
The proof is provided in Appendix \ref{proof:theo:sparse_2}. Essentially, to solve a full (\ref{A-MIND}) problem, we may partition $\displaystyle \vu$ into $q$ blocks, solve the $q$ (\ref{A-MIND}) problems in parallel using within-block constraints, determine whether $\prod_{i=1}^q p_{\text{AM}} \left(\displaystyle \vv_i; \eta_i,  \bm{\beta}_i \right)$ satisfies the between-blocks constraints to an acceptable tolerance, and if not solve the full (\ref{A-MIND}) problem using the parameters of $\prod_{i=1}^q p_{\text{AM}} \left(\displaystyle \vv_i; \eta_i,  \bm{\beta}_i \right)$ as initial parameters.

\section{Estimation}
\label{sct:estimation}
We now turn to estimating the parameters of the solutions of the problems (\ref{A-MIND}) and (\ref{MIND}). With $\phi_m^k$ satisfying (P5), let us consider the problem:
\begin{equation}
\label{CVX-MIND}
\min_{\bm{\theta}} ~~ -\bm{\theta}^T \bm{\alpha}_m^k + \int_{[0, 1]^d} e^{\bm{\theta}^T\phi_m^k \left( \displaystyle \vu \right)} d\displaystyle \vu. \tag{CVX-MIND}
\end{equation}
\subsection{Convex Estimation}
\begin{lemma}
\label{lem:cvx}
The optimization problem (\ref{CVX-MIND}) is strictly convex.
\end{lemma}
\begin{lemma}
\label{lem:same}
The minimizer of problem (\ref{CVX-MIND}) is the maximizer of problem (\ref{A-MIND}).
\end{lemma}
\begin{theorem}
\label{theo:fund}
Let $\phi_m^k$ satisfy (P3) and (P5). If $\bm{\hat{\beta}}_{m,n}$ is a consistent estimator of $\bm{\beta}_m := E_{P_{\displaystyle \vu_z}} \left[ \psi_m\left( \displaystyle \vu \right) \right]$, then for every $m>0$
\begin{align}
\label{eq:fund_1}
h_{\text{AM}} \left( \displaystyle \vu; \phi_m^k,  \bm{\hat{\beta}}_{m,n} \right) \underset{k,n \rightarrow\infty}{\longrightarrow}  h_{\text{M}}\left( \displaystyle \vu; \psi_m,  \bm{\beta}_m  \right), \tag{A}
\end{align}
and for every $k>0$
\begin{align}
\label{eq:fund_2}
h_{\text{AM}} \left(\displaystyle \vu; \phi_m^k,  \bm{\hat{\beta}}_{m,n} \right) \underset{m,n \rightarrow\infty}{\longrightarrow}  h\left( \displaystyle \vu_z \right). \tag{B}
\end{align}
\end{theorem}
See Appendix \ref{proof:lem:cvx} for the proof of Lemma \ref{lem:cvx}, Appendix \ref{proof:lem:same} for the proof of Lemma \ref{lem:same}, and Appendix \ref{proof:theo:fund} for the proof of Theorem \ref{theo:fund}.
\subsection{Choice of $k$ and $\psi_m$}
It follows from Theorem \ref{theo:fund} that $k$ controls how close to uniform marginals of $p_{\text{AM}}$ are, while $m$ controls how close $h_{\text{AM}}\left( \displaystyle \vu_z \right)$ is to $h\left( \displaystyle \vu_z \right)$. When the object of study is to learn the copula itself not just its entropy, $k$ should be as large as necessary. Note however that the maximum-entropy problem (\ref{A-MIND}) inherently favors distributions that are as close to uniform as allowed by empirical evidence, so that a large $k$ might not be needed in practice. 

An example family $\left(\psi_m\right)_m$ that satisfies (P3) are polynomials of degree $m$, thanks to the Stone-Weierstrass theorem (\cite{rudin1973functional}). With this choice of $\psi_m$, our approach can be regarded as a maximum-entropy Taylor expansion of the true log copula density $\log p\left( \displaystyle \vu_z \right)$. Recalling that $\rho_{ij} = 12E\left[u_iu_j\right]-3$ is the population version of the Spearman rank correlation between associated primal variables (\cite{nelsen2007introduction}), it follows that with $m$ as small as $2$, we can capture all smooth 1-to-1 associations\footnote{A continuous function of one variable is 1-to-1 if and only if it is either decreasing or increasing, both of which are captured by Spearman's rank correlation.} between any two coordinates of $\displaystyle \vz$. 

Monomials with degree $m>2$ grow combinatorially in number, but do not provide as much insights per term as $\psi_2$. Thus, before considering higher degree polynomials, we suggest leveraging Proposition \ref{prop:redund} in combination with $\psi_2$ to incorporate specific types of smooth but non-1-to-1 associations. For instance, using $f_{\bm{\mu}}(\displaystyle \vx) := \vert \displaystyle \vx -\bm{\mu} \vert$, where the absolute value is coordinatewise, and solving the problem (\ref{MIND}) with $m=2$ to estimate $I\left(\displaystyle \vy; \displaystyle \vx\right)$ through $I\left(\displaystyle \vy; \displaystyle \vx, f_{\bm{\mu}}(\displaystyle \vx) \right)$, allows us to reveal any possible associations of the type `a coordinate of $\bm{y}$ tend(s) to be monotonically related to the departure of coordinates of $\bm{x}$ from some baseline values', as well as all smooth 1-to-1 associations between coordinates of $\bm{x}$ and $\bm{y}$. Good examples for $\bm{\mu}$ are the sample median or mean of $\displaystyle \vx$. To capture departures from a standard range of values rather than a single one, $f_{\bm{\mu}}$ can be passed through an $\epsilon$-insensitive loss function. When output $y_j$ is a quasi-periodic function of $x_i$ (e.g. $x_i$ is time and output $y_j$ is seasonal), $f_{\pi}(x_i) = x_i - \lfloor x_i / \pi_i \rfloor \pi_i$ allows us to capture seasonality-adjusted effects.
\subsection{Handling Categorical Data}
Categorical and non-ordinal variables should be ordinarily encoded as customary, and ordinal data should be treated as continuous variables. The only practical requirements for the validity of this approach are i) to use a ranking function that assigns different ranks to all inputs including ties (e.g. scipy's `rankdata' function with method `ordinal'), and ii) to avoid encoding methods that may result in linearly dependent coordinates (e.g. one-hot-encoding on a binary non-ordinal categorical variable). When a suitable ranking function is not available a small random jitter may be added to ordinal variables to remove ties. 

This approach is mathematically valid thanks to the quantization characterization of the mutual information (\cite{incover1999elements}, Definition 8.54). See Appendix \ref{sct:handcat} for more details.
\subsection{Properties}
We summarize some key properties of our mutual information estimator $I_\text{AM}$.

\textbf{Non-negativity:} $I_\text{AM}$, is always non-negative thanks to requirement (P4) and Theorem \ref{theo:sparse_2}.

\textbf{Unboundedness for every} $n$: Because $I_\text{AM}$ only depends on estimated expected statistics, not i.i.d. samples themselves, and because said estimated expected statistics may take extreme values for any $n$, $I_\text{AM}$ cannot be upper-bounded by a function of $n$.

\textbf{Marginal-Invariance:} Our entire approach does not depend on marginal distributions. Additionally, our copula entropy estimator is invariant by any increasing univariate feature transformation (it depends on the data solely through ranks), and any smooth 1-to-1  univariate feature transformation for $k$ large enough enough (Proposition \ref{prop:inv} and Theorem \ref{theo:fund}-\ref{eq:fund_1}).

\textbf{Low Variance:} Our approach depends on the data generating distribution solely through the expected statistics constraints, which are estimated using Equation (\ref{eq:estm}) with $O(1/n)$ MSE rate (\cite{ruschendorf1976asymptotic}). By the delta method, both the associated natural parameters $\bm{\theta}$ and the corresponding copula entropies $-\bm{\theta}^T \bm{\hat{\alpha}}_{m, n}$, and therefore $I_\text{AM}$, have MSE rate $O(1/n)$.

\textbf{Consistency:} The fact that $I_\text{AM}$ is a consistent estimator of the true mutual information is a direct consequence of the consistency of the rank estimator Equation (\ref{eq:estm}) and of the individual copula entropy estimators as a result (see Theorem \ref{theo:fund}-\ref{eq:fund_2}).

\textbf{Low Complexity:} With our choice of $\phi_m$, pre-optimization complexity is dominated by the computation of ranks, which scales in $\mathcal{O}(d^2 n\log n)$, while optimization can scale in  $\mathcal{O}(d^2)$ using gradient descent, where $d$ is the number of inputs and output(s). Calculating the integral over $[0, 1]^d$ is only required while solving (\ref{CVX-MIND}) to compute the gradient and possibly the Hessian; it is not needed to calculate the optimal entropy itself. Thus, a crude approximation using naive Monte Carlo at every learning step is good enough; the resulting algorithm, which can be regarded as mini-batch stochastic gradient descent on a convex objective (\cite{bottou2010large}), will converge to the right solution even for large $d$.
\section{Applications}
\label{sct:application}
We begin by applying our approach to clarifying a common misconception. 
\subsection{The Multivariate Gaussian is Highly Structured}
The use of multivariate Gaussian variables is often justified by the fact that they are maximum-entropy (or the least informative of all distributions supported on $\mathbb{R}^d$) under Pearson correlation constraints. Such a choice is equivalent to assuming that marginals are Gaussian and the copula is the Gaussian copula. It might surprise the reader to know that the Gaussian copula is in fact highly structured/informative.
\begin{figure}[h]
\includegraphics[width=0.5\textwidth]{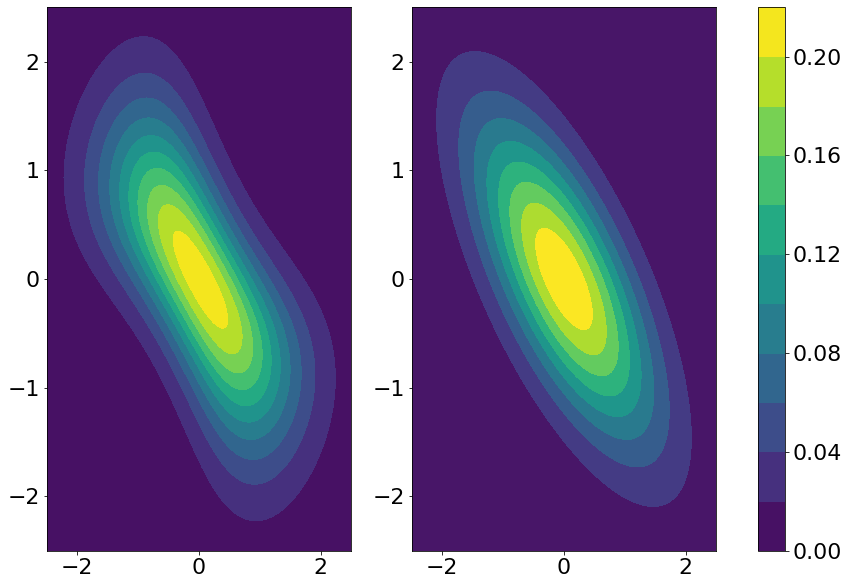}
\caption{Two bivariate pdfs with standard normal marginals, and the same Spearman rank correlation $-0.7$. The copula is Gaussian on the right, and maximum-entropy under the Spearman rank correlation constraint on the left.}
\label{fig:visual_pdf}
\end{figure}
For instance, it can be seen in Figure \ref{fig:visual_pdf} that the Gaussian copula posits that tails are much more tightly coupled than the corresponding\footnote{With the same Spearman correlation structure.}  maximum-entropy copula.

Additionally, as illustrated in Figure \ref{fig:distance} in the Appendix, the bivariate Gaussian pdf is on average about $10\%$ off, and up to $50\%$ off from the pdf with the same marginals and copula the least informative copula with the same Spearman correlation.\footnote{The Spearman rank correlation $\rho_s$ of a bivariate Gaussian with Pearson correlation $\rho$ reads $\rho_s = \frac{6}{\pi} \arcsin \left(\frac{\rho}{2} \right)$ (\cite{kruskal1958ordinal}).} When applying the maximum-entropy principle in the primal space, constraints should always be broken down into constraints that solely apply to the copula and constraints that solely apply to marginals (if any). If this is not done, marginal entropies will tend to dominate the copula entropy in the entropy decomposition of Proposition \ref{prop:entropy_decomp}, and the copula will tend to be low entropy. When constraints are so separable, Equation (\ref{eq:ent_decomp}) allows us to break down the optimization problem into two, one maximum-entropy problem about the copula, which our approach allows solving, and one about marginals. This is not possible using covariance matrices as maximum-entropy constraints given that Pearson's correlation is not a functional of the copula: it \emph{does} depend on marginals. As expected, marginal entropies in this case do dominate the copula entropy in the maximum-entropy problem, which explains why univariate Gaussians are indeed high entropy but the Gaussian copula is low entropy/highly structured.
\begin{figure}[h]
\includegraphics[width=0.5\textwidth]{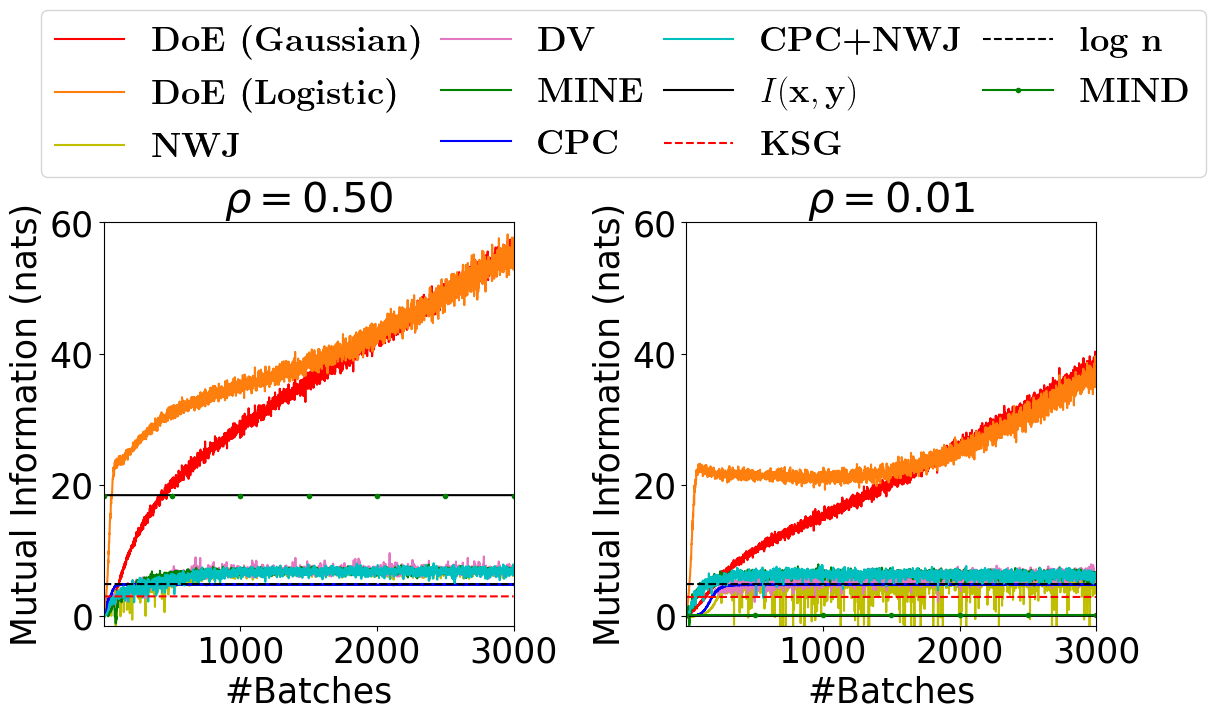}
\caption{Estimation of the mutual information between two $128$-dimensional vectors $\bm{x}=(x_1, \dots, x_d)$ and $\bm{y}=(y_1, \dots, y_d)$ from a draw of $1000$ i.i.d. samples. $(x_i, y_i)$ are i.i.d. Gaussians with mean zero, unit marginal variance, and correlation $\rho$. For models using deep neural networks, we run the experiment $10$ times and report for each batch number the estimate that is the closest to the ground truth. The ground truth is the black solid line labeled $I\left(\bm{x}; \bm{y}\right)$.}
\label{fig:comp}
\end{figure}
\subsection{MIND is Far More Data-Efficient Than Competing Approaches}
Next, we illustrate that our approach is far more data-efficient than all alternatives, in both low and high mutual information settings. 

We repeat the experiment of \cite{mcallester2020formal}, and estimate the mutual information  between two $d$-dimensional vectors $\bm{x}=(x_1, \dots, x_d)$ and $\bm{y}=(y_1, \dots, y_d)$, where  $(x_i, y_i)$ are i.i.d. Gaussians with mean zero, unit marginal variance, and correlation $\rho$. The true mutual information in this case is $I(\bm{y}; \bm{x}) = -\frac{d}{2}\log \left(1-\rho^2 \right)$. We reuse the exact same settings as \cite{mcallester2020formal},  except for one simple change. Rather than drawing a fresh mini-batch from the true data generating distribution, which is equivalent to using $384000$ i.i.d. samples in total, we generate $1000$ i.i.d. samples used by all experiments and from which mini-batches are sampled. We run the experiment in a high ($\rho=0.5$) and a low ($\rho=0.01$) mutual information setting. We use the code provided by the authors of \cite{mcallester2020formal} at \url{https://github.com/karlstratos/doe} for all models but KSG (\cite{kraskov2004estimating}) and MIND. For MIND, we use second order polynomials as $\phi_m$. As it can be seen in Figure \ref{fig:comp} and in Table \ref{tab:comp}, MIND is the only model able to come anywhere close to the ground truth in  high or low mutual information settings. DoE models clearly overshoot in both settings. If we refer to Figure 2 of \cite{mcallester2020formal}, we may  conclude that DoE models need about $2000 \times 128$ i.i.d. samples to converge to the ground truth in this experiment, which is $256$ times more than what MIND requires. Variational models in the primal space overshoot in low mutual information settings and seem to be upper-bounded by $O(\log n)$ in high mutual information settings. As for the nonparametric KSG estimator, it struggles with large input dimensions.
\begin{table}
\centering
\begin{tabular}{ c ||c | c | c | c | c}
                               $I(\bm{y}; \bm{x})$ &  MIND & DoE & MINE & NWJ& KSG \\ 
\hline 
                               18.41 &  \bf{18.36} & 53.88 & 6.92 &  7.22 & 2.98\\ 
\hline 
                              0.01 &  \bf{0.08} & 37.49 & 5.83 & 1.67 & 2.89 
\end{tabular}
\caption{Estimates of the mutual information between two $128$-dimensional vectors from $1000$ i.i.d. samples in low and high mutual information settings using various models.  The ground truth is the column $I(\bm{y}; \bm{x})$ and the closest model to it is in bold.}
\label{tab:comp}
\end{table}
\begin{figure*}
\centering
\includegraphics[width=0.3\textwidth]{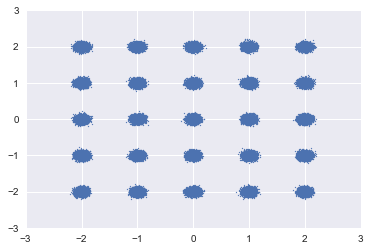}
\includegraphics[width=0.3\textwidth]{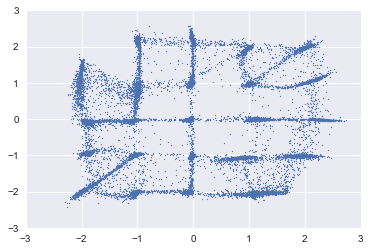}
\includegraphics[width=0.3\textwidth]{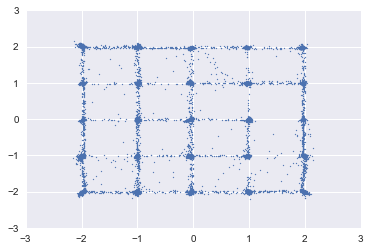}
\caption{Illustration of GANs trained on the 25 Gaussians dataset (left) without regularization (middle), and with copula entropy regularization (right). Both models use the same discriminator and generator architectures and were stopped after 500 epochs.}
\label{fig:25_gaussians}
\end{figure*}
\subsection{Copula Entropy Regularized Generative Adversarial Networks (CER-GANs)}
Finally, we illustrate that our approach may be used to prevent mode collapse in GANs.

We recall that GANs (\cite{goodfellow2014generative}) are very effective generative models made of two modules: a generator whose state is represented by a function $G: \mathcal{Z} \subset \mathbb{R}^q \to \mathcal{X}$, and a discriminator whose state is represented by a function $D: \mathcal{X} \to [0, 1]$. The aim is to learn a function $G$ that maps a simple noise or code distribution $P_{\bm{z}}$ supported on $\mathcal{Z}$ to a true data generating distribution of interest $P_{\bm{x}}$ supported on $\mathcal{X}$, so that we may draw samples from the true distribution $P_{\bm{x}}$ (e.g. realistic images) simply as $G(\bm{z})$ with $\bm{z} \sim P_{\bm{z}}$. To do so, a GAN alternates between two steps. The first step consists of learning a discriminating function $D$ that, as the predictive probability of a binary classifier, is effective at telling draws from the true distribution $P_{\bm{x}}$ apart from draws of the form $G(\bm{z})$ (i.e. that are fake). This is done by maximizing the likelihood-like objective
\begin{align*}
\mathcal{L}_D := E_{P_{\bm{x}}}\left[\log D\left(\bm{x}\right) \right] +  E_{P_{\bm{z}}}\left[ \log \left(1- D\left(G(\bm{z}) \right) \right) \right]
\end{align*}
over functions $D$ induced by a deep neural network, for a given $G$. The second step consists of updating the state $G$ of the generator so as to trick the discriminator into thinking that fake samples are real. This is done by minimizing the objective
\begin{align*}
\mathcal{L}_G := E_{P_{\bm{z}}}\left[ \log \left( 1- D\left(G(\bm{z}) \right)  \right) \right]
\end{align*}
over functions $G$ induced by another deep neural network. When $P_{\bm{x}}$ is multi-modal, as is often the case in real-life applications, if the generator becomes good at sampling from a mode of $P_{\bm{x}}$, then it will keep generating samples near the same mode, as no term in the objectives $\mathcal{L}_D $ and $\mathcal{L}_G$ incentivizes the generator to keep exploring beyond a mode. This pathology of GANs is known as mode collapse (\cite{che2019mode}). \cite{belghazi2018mutual} identified as possible solution regularizing $\mathcal{L}_G$ with the entropy of $G(\bm{z})$ so as to foster exploration, but the authors considered this solution intractable. Instead they followed \cite{chen2016infogan} and focused on cases where, in addition to $\bm{z}$, the generator uses meta-data $\bm{c}$ that implicitly identify modes of the true distribution $P_{\bm{x}}$ (e.g. the digit in the case of MNIST). The authors realized that a low mutual information $I\left(G(\bm{z}, \bm{c}); \bm{c} \right)$ between fake samples and the associated codes reveals mode collapse, and consequently proposed regularizing $\mathcal{L}_G$ with the negative of the foregoing mutual information term:
\begin{align*}
\tilde{\mathcal{L}}_G = \mathcal{L}_G -\beta I\left(G(\bm{z}, \bm{c}); \bm{c} \right), ~ \beta > 0.
\end{align*}
We propose an alternative that does not rely on meta-data. 

It follows from Sklar's theorem (Theorem \ref{theo:sklar}) that to prevent the generator from collapsing to a mode, it suffices to prevent the copula of $G(\bm{z})$ from collapsing to a mode, which can be done by regularizing the generator's loss function with the entropy of the copula of $G(\bm{z})$:
\begin{align*}
\bar{\mathcal{L}}_G = \mathcal{L}_G -\beta h\left(\bm{u}_{G(\bm{z})} \right), ~ \beta > 0.
\end{align*}
Using the problem (\ref{CVX-MIND}), we may write the regularized batch training step of the generator as
\begin{align*}
 \min_{G} ~~ \frac{1}{b} \sum_{i=1}^b  &\log \left[1-D\left(G(\bm{z}_i)\right) \right] + \beta \left[ \bm{\theta}^T \phi \left( \hat{\bm{u}}_i\right) -  e^{\bm{\theta}^T\phi \left( \displaystyle \vu_i \right)} \right]
\end{align*}
where $\beta > 0$, $\phi$ is the constraint function, $\bm{u}_i \sim \mathcal{U}\left([0, 1]^d\right)$ are standard uniforms, and $\hat{\bm{u}}_i = \frac{\text{rg}(G(\bm{z}_i))}{b+1}$, the rank being understood as within-batch. The generator step is now followed by the batch copula step
\begin{align*}
 \min_{\bm{\theta}} ~~ \frac{1}{b} \sum_{i=1}^b  -\bm{\theta}^T \phi \left( \hat{\bm{u}}_i\right) +  e^{\bm{\theta}^T\phi \left( \displaystyle \vu_i \right)}.
\end{align*}
We call this model Copula Entropy Regularized GANs (CER-GANs). Figure \ref{fig:25_gaussians} illustrates the efficacy of CER-GANs on the 25-Gaussians dataset ( \cite{belghazi2018mutual}).
\section{Conclusion}
We propose a novel approach for solving maximum-entropy copula problems under flexible linear constraints as a convex optimization problem, and we apply our finding to estimating the mutual information between two random vectors. Our approach is inductive in that it relies on the data generating distribution solely through some association-revealing nonparametric properties; it does not assume we have enough data to fully characterize the underlying true joint distribution. This allows the resulting estimator, which we denote MIND, to be considerably more data efficient than all competing models. We show that MIND can accurately estimate the mutual information even when the sample size $n$ is not large enough to fully characterize the true data generating distribution. For large $n$ settings, we show that MIND is a consistent estimator of the true mutual information and has MSE rate $O(1/n)$. Beyond mutual information estimation, we show that mode collapse in GANs can be mitigated by adding a regularizing term that maximizes  the copula entropy of the generator using MIND.

An implementation of MIND can be accessed from the Python package `kxy' available from Pypi (by running `pip install kxy') or GitHub (\url{https://github.com/kxytechnologies/kxy-python}).
\bibliography{mind}
\bibliographystyle{mind}
\newpage
\appendix

\begin{table*}[t]
\centering
\resizebox{\textwidth}{!}{
\begin{tabular}{ c |c |c }
\hline 
                               & $y$ \bf{ is continuous} & $y$ \bf{is categorical} \\ 
\hline 
$\displaystyle \vx$ \bf{is continuous} & $I(\bm{y}; \displaystyle \vx) = h(\bm{y}) + h(\displaystyle \vx) - h(\bm{y}; \displaystyle \vx)$ & $I(y; \displaystyle \vx) = h(\displaystyle \vx) - \sum_{i\in \mathcal{Y}} h(\displaystyle \vx | y=i)P_y(i)$ \\  
\hline 
$\displaystyle \vx$ \bf{is categorical} & $I(\bm{y}; \displaystyle \vx) = h(\bm{y}) - \sum_{i\in \mathcal{X}} h(\bm{y} | \displaystyle \vx=i)P_{\displaystyle \vx}(i)$ & $I(y; \displaystyle \vx) = H(y) + H(\displaystyle \vx) - H(y; \displaystyle \vx)$ \\
\hline 
\vtop{\hbox{\strut $\displaystyle \vx$ \bf{has continous  coordinates} $\displaystyle \vx_c$}\hbox{\strut \bf{and categorical coordinates} $\displaystyle \vx_d$}}& $I(\bm{y}; \displaystyle \vx) = h(\bm{y}) + \sum_{i \in \mathcal{X}_d} \left[h(\displaystyle \vx_c | \displaystyle \vx_d = i) -  h\left(\bm{y}, \displaystyle \vx_c | \displaystyle \vx_d = i\right)\right]P_{\displaystyle \vx_d}(i)$ & $I(y; \displaystyle \vx) = I(y; \displaystyle \vx_d) + \sum_{i \in \mathcal{X}_d} P_{\displaystyle \vx_d}(i) h(\displaystyle \vx_c | \displaystyle \vx_d = i) -  \sum_{j \in \mathcal{Y}} h\left(\displaystyle \vx_c | \displaystyle \vx_d = i, y=j\right)P_{\displaystyle \vx_d, y}(i, j)$ \\
\hline 
\end{tabular}}
\caption{Expression of the mutual information $I(y; \displaystyle \vx)$ as a function of the Shannon entropy $H(.)$, and/or the differential entropy $h(.)$, depending on whether $y$ and/or $\displaystyle \vx$ has continuous and/or categorical coordinates. Expressions of the type $h(\displaystyle \vx|y=i)$ are to be understood as the differential entropy of the continuous conditional distribution $\displaystyle \vx|y=i$.}
\label{tab:mutual_information}
\end{table*}
\begin{figure}[h]
\includegraphics[width=0.5\textwidth]{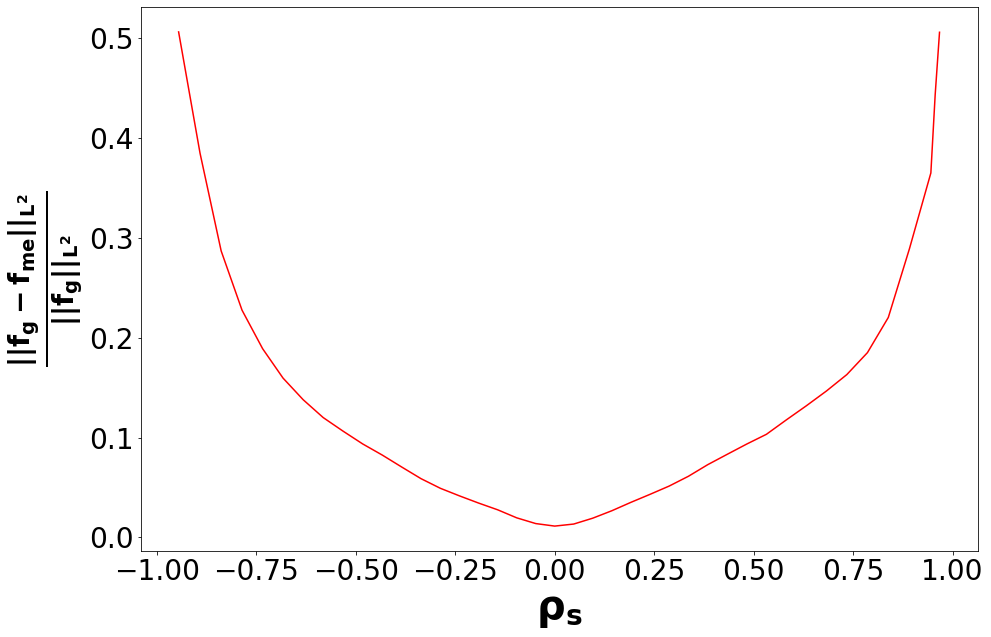}
\caption{Normalized $L_2$ distance between two bivariate distributions with standard normal marginals, one with Gaussian copula and the other with maximum-entropy copula under Spearman correlation constraints, for various Spearman correlation values $\rho_s$.}
\label{fig:distance}
\end{figure}
\section{Further Details}
\subsection{Additional Experiment}
\label{sct:addexp}
It can be shown that the highest $R^2$ and the lowest RMSE achievable by a regression model using $\bm{x}$ to predict $y$ read respectively
$$\bar{R}^2\left( P_{\bm{x}, y} \right) = 1- e^{-2I(y; \bm{x})}$$ 
and 
$$\bar{RMSE}\left( P_{\displaystyle \vx, y} \right) = e^{-I\left(y ; \displaystyle \vx \right)} \sqrt{\mathbb{V}\text{ar}\left( y \right)}.$$
We use MIND to estimate the highest performance achievable in the Kaggle competition `House Prices: Advanced Regression Techniques'. The aim is to predict the price at which houses were sold from various continuous and categorical variables. The results are summarized in Table \ref{tab:kaggle}, when all variables are used, when the $10$ explanatory variables OverallQual, GrLivArea, YearBuilt, TotalBsmtSF, OverallCond, LotArea, BsmtFinSF1, BldgType, KitchenQual, MSZoning are used, and when the first $5$ of the foregoing list are used.
\begin{table}[h]
\centering
\begin{tabular}{ c |c | c | c | c }
                             $I(y; \bm{x})$  & $\bar{R}^2$ &  $\bar{RMSE}$ &  $d$ & $n$ \\ 
\hline 
                             1.50 & 0.95 &  18,531 & 80 &  1460 \\ 
\hline 
                           0.76  & 0.78 &  36,979 & 10 &  1460 \\ 
\hline 
                          0.65  & 0.73 &  41,007 & 5 &  1460 
\end{tabular}
\caption{Mutual information and highest performances achievable in the `House Prices: Advanced Regression Techniques' Kaggle challenge using all, 10 or 5 explanatory variables.}
\label{tab:kaggle}
\end{table}
\subsection{Further Details on Handling Categorical Data}
\label{sct:handcat}
Categorical and non-ordinal variables should be ordinarily encoded as customary, and ordinal data should be treated as continuous variables. The only practical requirements for the validity of this approach are i) to use a ranking function that assigns different ranks to all inputs including ties (e.g. scipy's `rankdata' function with method `ordinal', or leveraging PyTorch's or Tensorflow's `argsort'.), and ii) to avoid encoding methods that may result in linearly dependent coordinates (e.g. one-hot-encoding on a binary non-ordinal categorical variable). When a suitable ranking function is not available a small random jitter may be added to ordinal variables to remove ties.

This approach is mathematically valid thanks to the quantization characterization of the mutual information (\cite{incover1999elements}, Definition 8.54). In effect, if we denote $\mathcal{P}$ (resp. $\mathcal{Q}$) a partition of the domain $\mathcal{X}$ (resp. $\mathcal{Y}$) of $\bm{x}$ (resp. $\bm{y}$), and $[\bm{x}]_{\mathcal{P}} \in \mathbb{N}$ (resp. $[\bm{y}]_{\mathcal{Q}} \in \mathbb{N}$) a discrete random variable indicating which element of $\mathcal{P}$ (resp. $\mathcal{Q}$) $\bm{x}$ (resp. $\bm{y}$) belongs to, then we have
\begin{equation}
I\left(\bm{y}; \bm{x}\right) = \sup_{\mathcal{P}, \mathcal{Q}} I\left( [\bm{y}]_{\mathcal{Q}}; [\bm{x}]_{\mathcal{P}} \right),
\end{equation}
where the rightmost mutual information is between discrete random variables.

This characterization implies that the mutual information is always invariant by 1-to-1 transformations, of which ordinal encoding is one, whether coordinates are all continuous, all categorical, or a mix.\footnote{Indeed, for any 1-to-1 transformation $f$ and partition $\mathcal{P}$ of $\mathcal{X}$, we may always find a partition $\mathcal{P}^\prime$ of the image space $f\left(\mathcal{X}\right)$ such that $[\bm{x}]_{\mathcal{P}}=[f\left(\bm{x}\right)]_{\mathcal{P}^\prime}$.} It can also be seen that adding a negligible random jitter to an ordinal random variable will not materially change the mutual information,\footnote{If we denote $g$ the operation consisting of adding a negligible random jitter to an ordinal variable $x_i$, then by reducing the jitter's standard deviation, for any partition $\mathcal{P}$ of the domain of $x_i$, we may always find a partition $\mathcal{P}^\prime$ of the image space such that $\mathbb{P}\left( [x_i]_{\mathcal{P}}=[g\left(x_i\right)]_{\mathcal{P}^\prime} \right) = 1- e$ for $e$ arbitrarily small.} but will turn the ordinal variable into a continuous one so that results developed for continuous variables may apply. Strictly speaking, by the data processing inequality (\cite{incover1999elements}, Theorem 2.8.1), adding a random jitter to ordinal variables increases the mutual information but, as the jitter standard deviation goes to zero, the difference becomes negligible, even though we still enjoy the benefits of working with continuous variables, without downside. Considering that MIND only depends on variables through their ranks, we may do without adding a jitter, so long as the ranking algorithm does not attribute the same rank to ties.

Another approach for handling categorical variables would be to use Table \ref{tab:mutual_information}, and to use the entropy decomposition formula (Equation (\ref{eq:ent_decomp})) to estimate differential entropies and conditional entropies. However, this approach can be far less data-efficient as it requires splitting the dataset into as many subsets as the number of distinct tuples of categorical variable values, so as to evaluate conditional differential entropies.
\section{Proofs}
\subsection{Proof of Proposition \ref{prop:inv}}
\label{proof:prop:inv}
Any continuous 1-to-1 univariate transformation $g_i$ is either increasing or decreasing. Increasing transformations leave copulas  invariant. Moreover, any continuous decreasing function $g_i$ on $\mathbb{R}$ can be written as $f(-x)$ where $f$ is a continuous increasing function, so that we may focus on proving the statement for the transformation $g_i: z \to -z$. Let us denote $\bar{\displaystyle \vz} = \left(z_1, \dots, -z_i, \dots, z_d \right)$, $\bar{c}$ its copula density, and $c$ the copula density of $\displaystyle \vz = \left(z_1, \dots, z_d \right)$. We have $$\bar{c}\left(u_1, \dots, u_d\right) = c\left( u_1, \dots, 1-u_i, \dots, u_d \right).$$ A simple change of variables shows that $h\left(\displaystyle \vu_z\right) = h\left( \displaystyle \vu_{\bar{z}}\right)$.

\subsection{Proof of Theorem \ref{theo:MIND}}
\label{proof:theo:MIND}
Let $P$ be a $d$-dimensional copula distribution, and $U$ the uniform distribution on $[0, 1]^d$. We note that $h(P)=-KL(P||U).$ Thus the optimization problem (\ref{MIND}) is equivalent to looking for the $I$-projection of $U$ on the space $\mathcal{E}$ of copula distributions satisfying the linear constraint $E_P\left[ \phi_m(\displaystyle \vu) \right]= \bm{\alpha}_m$, as defined in \cite{csiszar1975divergence}. 

\textbf{Existence and uniqueness}: If there exists a copula distribution $P$ satisfying the constraints and admitting an entropy, then $\mathcal{E}$ is not empty. $\mathcal{E}$ is convex as every convex combination of copulas satisfying the linear constraint $E_P\left[ \phi_m(\displaystyle \vu) \right]= \bm{\alpha}_m$ is itself a copula that satisfies said constraint. 

We say that a space of continuous distributions supported on $[0, 1]^d$ is variation closed when it is closed in the topology of the variation distance $\vert P - Q \vert = \int |p - q| dU$, where $p$ and $q$ are the Radon-Nikodym derivatives of $P$ and $Q$ with respect to $U$ (i.e. their pdfs).

\begin{lemma}
$\mathcal{E}$ is variation-closed.
\end{lemma}
\begin{proof}
Let $P_n \in \mathcal{E}$ be a sequence converging in variation to a distribution $P$. We need to show that $P$ also satisfies the linear constraints and has uniform marginals. Convergence in variation implies that for every test function $f$
\begin{equation}
\label{eq:proof:var}
\int_{[0, 1]^d} f(\displaystyle \vu) p_n(\displaystyle \vu) d\displaystyle \vu \to \int_{[0, 1]^d} f(\displaystyle \vu) p(\displaystyle \vu) d\displaystyle \vu.
\end{equation}
Taking $f=\phi_m$ proves that the limit distribution $P$ satisfies the linear constraints. We now need to prove that it has uniform marginals.

Let us consider $\displaystyle \vu_{-i} = (\dots, u_{i-1},u_{i+1}, \dots)$ the vector $\displaystyle \vu$ without its $i$-th coordinate, and let us choose a test function $f$ that only depends on $u_i$. By Fubini's theorem we have 
\begin{align*}
\int_{[0, 1]^d} f(\displaystyle \vu) p_n(\displaystyle \vu) d\displaystyle \vu &= \int f(u_i) \underbrace{\left( \int p_n(\displaystyle \vu) d\displaystyle \vu_{-i} \right)}_{=1} du_i \\
&= \int_{[0,1]} f(u_i) du_i \\
&= \int_{[0, 1]^d} f(u_i) p(\displaystyle \vu) d\displaystyle \vu
\end{align*}
where the last equality is due to the Equation (\ref{eq:proof:var}). Putting the last two equality together, we get $$\int f(u_i) \left( 1 - \int p(\displaystyle \vu) d\displaystyle \vu_{-i} \right) du_i = 0$$
for every univariate test function $f$, which implies that every marginal of $P$ is uniform.
\end{proof}
$\mathcal{E}$ being convex, non-empty, and variation closed, $U$ admits a unique $I$-projection on $\mathcal{E}$ (see Theorem 2.1 in \cite{csiszar1975divergence}), or equivalently, (\ref{MIND}) admits a unique solution.

\textbf{Functional form of the pdf}: Let us denote $\mathcal{F}$ the space of distributions supported on $[0, 1]^d$ that satisfy the linear constraint $E_P\left[ \phi_m(\displaystyle \vu) \right]= \bm{\alpha}_m$. Clearly, $\mathcal{E} \subset \mathcal{F}$ as the only difference between the two sets is that $\mathcal{E}$ only contains elements of $\mathcal{F}$ with uniform marginals (i.e. copula distributions). The $I$-projection $P_{\text{AM}}$ of $U$ on $\mathcal{F}$, which exists because $\mathcal{F}$ is convex, non-empty, and variation cloosed, is the minimizer of the problem (\ref{A-MIND}).

By Theorem 2.3 in \cite{csiszar1975divergence}, $P_{\text{M}}$ is the $I$-projection of $P_{\text{AM}}$ on $\mathcal{E}$. 

Applying Theorem 3.1 (Case A) in \cite{csiszar1975divergence}, we get that $P_{\text{AM}}$ has density with respect to $U$, which is also its pdf, of the form $p_{\text{AM}} = e^{\bm{\theta}^T\phi_m(\displaystyle \vu)}$. Moreover, any distribution in $\mathcal{F}$ with density with respect to $U$ of this form is the $I$-projection of $U$ on $\mathcal{F}$.

Applying Theorem 3.1 (Case B) in \cite{csiszar1975divergence}, we get that $P_{\text{M}}$ has density with respect to $P_{\text{AM}}$ of the form $\prod_{i=1}^d f_i\left(u_i\right)$, where $f_i$ are non-negative and log-integrable. Moreover, any distribution in $\mathcal{E}$ with density with respect to $P_{\text{AM}}$ of this form is the $I$-projection of $P_{\text{AM}}$ on $\mathcal{E}$.

Hence, $P_{\text{M}}$ has density with respect to $U$, which is also its pdf,
$$p_{\text{M}} = e^{\bm{\theta}^T\phi_m(\displaystyle \vu)} \prod_{i=1}^d f_i\left(u_i\right),$$
and any distribution on $[0, 1]^d$ whose pdf of this form is the minimizer of (\ref{MIND}).

\textbf{Pythagoras' Identity}: 
Theorem 3.1 in \cite{csiszar1975divergence} guarantees that identity (3.1) in \cite{csiszar1975divergence} holds and
\begin{align}
-h(P) = -h\left( P_{\text{AM}} \right) +  KL\left( P \vert \vert P_{\text{AM}} \right)
\end{align}
for any $P \in \mathcal{F}$ and 
\begin{align}
KL( Q || P_{\text{AM}} ) =& KL( P_{\text{M}} || P_{\text{AM}} ) +  KL\left( Q \vert \vert P_{\text{M}} \right) 
\end{align}
for any $Q \in \mathcal{E} \subset \mathcal{F}$.

As $Q \in \mathcal{F}$, $$KL( Q || P_{\text{AM}} )  = h\left( P_{\text{AM}}\right) -h(Q).$$ Thus,
\begin{align*}
h\left( P_{\text{AM}}\right) -h(Q) =&KL( P_{\text{M}} || P_{\text{AM}} ) +  KL\left( Q \vert \vert P_{\text{M}} \right) 
\end{align*}
and
\begin{align*}
-h(Q) =& - h\left( P_{\text{AM}}\right) + KL( P_{\text{M}} || P_{\text{AM}} ) +  KL\left( Q \vert \vert P_{\text{M}} \right).
\end{align*}
As $P_{\text{M}} \in  \mathcal{F}$, 
\begin{align*}
-h\left( P_{\text{M}} \right) =& - h\left( P_{\text{AM}}\right) + KL( P_{\text{M}} || P_{\text{AM}} ),
\end{align*} and we get
\begin{align*}
-h(Q) &= -h( P_{\text{M}}) +  KL\left( Q \vert \vert P_{\text{M}} \right).
\end{align*}

\subsection{Proof of Theorem \ref{theo:a-MIND}}
\label{proof:theo:a-MIND}
To prove Theorem \ref{theo:MIND} we had to prove Theorem \ref{theo:a-MIND}. See Section \ref{proof:theo:MIND}.

\subsection{Proof of Theorem \ref{theo:sparse_1}}
\label{proof:theo:sparse_1}
Let $$g\left(\displaystyle \vu; \phi_m,  \bm{\alpha}_m \right) := \prod_{i=1}^q p_{\text{M}} \left(w, \displaystyle \vv_i; \eta_i,  \bm{\beta}_i \right).$$ We want to prove that $$g\left(\displaystyle \vu; \phi_m,  \bm{\alpha}_m \right) = p_{\text{M}}\left(\displaystyle \vu; \phi_m,  \bm{\alpha}_m \right).$$

We know from Theorem \ref{theo:MIND} that $p_{\text{M}} \left(\displaystyle \vu; \phi_m,  \bm{\alpha}_m \right)$ is the only copula density of the form $$e^{\bm{\theta}^T \phi_m(\displaystyle \vu)}  = e^{\sum_{i=1}^q \bm{\theta}_i^T \eta_i\left(w, \displaystyle \vv_i\right)} = \prod_{i=1}^q e^{\bm{\theta}_i^T \eta_i\left(w, \displaystyle \vv_i\right)}$$ that satisfies the constraints $E_P\left[ \psi_m \left(\displaystyle \vu \right)  \right] = \bm{\beta}_m$.

First we note that $g$ is a copula entropy. Indeed, integrating $g$ with respect to every variable but a coordinate of $\displaystyle \vv_i$ is always $1$ by virtue of the fact that $p_{\text{M}} \left(w, \displaystyle \vv_i; \eta_i,  \bm{\beta}_i \right)$ are copula densities. To see why, note that we may first integrate with respect to $\displaystyle \vv_j$ for all $j\neq i$, and then with respect to $\omega$ and all other coordinates of $\displaystyle \vv_i$. Additionally, if we integrate with respect to all variables but $\omega$, we get 
\begin{align*}
&\int g\left(\displaystyle \vu; \phi_m,  \bm{\alpha}_m \right) d\displaystyle \vv_1 \dots d\displaystyle \vv_q \\
&= \int p_{\text{M}} \left(w, \displaystyle \vv_q; \eta_q,  \bm{\beta}_q \right) \times \dots \times \\
& ~\left( \underbrace{\int p_{\text{M}} \left(w, \displaystyle \vv_1; \eta_1,  \bm{\beta}_1 \right) d\displaystyle \vv_1 }_{=1 }\right) d\displaystyle \vv_2 \dots d\displaystyle \vv_q \\
&= 1.
\end{align*}
Second, $g$ clearly has the form $\prod_{i=1}^q e^{\bm{\theta}_i^T \eta_i\left(w, \displaystyle \vv_i\right)}$.

Finally, $g$ satisfies the constraints $E_P\left[ \psi_m \left(\displaystyle \vu \right)  \right] = \bm{\beta}_m$ as
\begin{align*}
&\int \eta_i\left(w, \displaystyle \vv_i \right) g\left(\displaystyle \vu; \phi_m,  \bm{\alpha}_m \right) d\displaystyle \vu \\
=& \int \eta_i\left(w, \displaystyle \vv_i \right) p_{\text{M}} \left(w, \displaystyle \vv_i; \eta_i,  \bm{\beta}_i \right) \\
& \times \left( \underbrace{\int \prod_{j \neq i}  p_{\text{M}} \left(w, \displaystyle \vv_j; \eta_j,  \bm{\beta}_j \right) d\displaystyle \vv_j}_{=1} \right) d\omega d\displaystyle \vv_i \\
=& \int \eta_i\left(w, \displaystyle \vv_i \right) p_{\text{M}} \left(w, \displaystyle \vv_i; \eta_i,  \bm{\beta}_i \right) d\omega d\displaystyle \vv_i  \\
=& \bm{\beta}_i,
\end{align*}
where we've used the fact that each $p_{\text{M}} \left(w, \displaystyle \vv_j; \eta_j,  \bm{\beta}_j\right)$ has uniform marginals, and satisfies the constraint $$E_P\left[ \eta_j \left(\omega, \displaystyle \vv_j \right)  \right] = \bm{\beta}_j.$$

\subsection{Proof of Theorem \ref{theo:sparse_2}}
\label{proof:theo:sparse_2}
The essence of the proof is in the transitivity property of $I$-projections. Indeed, if $\mathcal{P} \subset \mathcal{Q}$ are linear sets of probability distributions supported on $[0, 1]^d$, $Q$ the $I$-projection of the standard uniform $U$ on $\mathcal{Q}$, and $P$ the $I$-projection of $U$ on $\mathcal{P}$, then $P$ is also the $I$-projection of $Q$ on $\mathcal{P}$ (Theorem 2.3 \cite{csiszar1975divergence}).

In the case of Theorem \ref{theo:sparse_2}, $\mathcal{Q}$ is the set of probability distributions satisfying all constraints, and $\mathcal{Q}$ is the set of probability distributions satisfying all but the between-blocks constraints. If we denote, $q\left(\displaystyle \vu; \psi_m,  \bm{\beta}_m \right) := \prod_{i=1}^q p_{\text{AM}} \left(\displaystyle \vv_i; \eta_i,  \bm{\beta}_i \right)$,  a direct application of Theorem \ref{theo:fund} shows that $q$ is the density of the $I$-projection $Q$ of $U$ on $\mathcal{Q}$. The maximizer of the full (\ref{A-MIND}) problem is therefore the $I$-projection of $Q$ on $\mathcal{P}$, and we know from Theorem 3.1 in \cite{csiszar1975divergence} that it has Radon-Nikodym derivative with respect to $Q$ of the form $\frac{dP}{dQ} =e^{\bm{\theta}^T\phi_m^k\left( \bm{u}\right)}.$ Putting everything together, we get that the maximizer $P$ of the full (\ref{A-MIND}) problem has pdf of the form $$ p_{\text{AM}} \left(\displaystyle \vu; \phi_m,  \bm{\beta}_m \right) = e^{\bm{\theta}^T \phi_m^k\left(\displaystyle \vu\right)}\prod_{i=1}^q p_{\text{AM}} \left(\displaystyle \vv_i; \eta_i,  \bm{\beta}_i \right).$$
The unicity of this representation is a direct consequence of Theorem \ref{theo:fund}. Taking the negative $\log$ of this expression and then the expectation, we get 
\begin{align*}
&h_{\text{AM}} \left(\displaystyle \vu; \phi_m^k,  \bm{\beta}_m \right) = -\bm{\theta}^T \bm{\alpha}_m^k \\
&- \sum_{i=1}^q E_{p_{\text{AM}} \left(\displaystyle \vv_i; \phi_m^k,  \bm{\beta}_m \right)} \left[ \log p_{\text{AM}} \left(\displaystyle \vv_i; \eta_i,  \bm{\beta}_i \right) \right].
\end{align*}
The final result stems from the identity $$E_p\left( - \log q \right) = E_p\left(-\log p \right) - \text{KL}\left[ p || q \right].$$ 
The fact that $$-\bm{\theta}^T\bm{\alpha}_m^k \leq \sum_{i=1}^q \text{KL}\left[ p_{\text{AM}} \left(\displaystyle \vv_i; \phi_m^k,  \bm{\beta}_m \right) \vert \vert p_{\text{AM}} \left(\displaystyle \vv_i; \eta_i,  \bm{\beta}_i \right)  \right]$$ is a direct consequence of $\mathcal{P} \subset \mathcal{Q}$, which implies that the entropy of $P$ cannot be greater than that of $Q$, and the fact that $\sum_{i=1}^q h_{\text{AM}} \left(\displaystyle \vv_i; \eta_i,  \bm{\beta}_i \right)$ is the entropy of $Q$. The two entropies are the same if and only if $P=Q$ or, equivalently, $\bm{\theta} = 0$. When $\forall \bm{u}, ~ \gamma\left( \bm{u} \right)$ is constant and equal to $1$, $P=Q$.

\subsection{Proof of Lemma \ref{lem:cvx}}
\label{proof:lem:cvx}
The Hessian of the objective, namely $\int_{[0, 1]^d} \left( \phi_m \left( \displaystyle \vu \right) \phi_m \left( \displaystyle \vu \right)^T \right) e^{\bm{\theta}^T\phi_m \left( \displaystyle \vu \right)} d\displaystyle \vu$, is clearly strictly positive-definite as coordinates of $1$ are not linearly related.

\subsection{Proof of Lemma \ref{lem:same}}
\label{proof:lem:same}
The only critical point of the objective of (\ref{CVX-MIND}) satisfies $$\bm{\alpha}_m = \int_{[0, 1]^d} \phi_m \left( \displaystyle \vu \right) e^{\bm{\theta}^{*T}\phi_m \left( \displaystyle \vu \right)} d\displaystyle \vu.$$ Given that the first coordinates of $\phi_m$ and $\bm{\alpha}_m$ are both $1$, $\int_{[0, 1]^d} e^{\bm{\theta}^{*T}\phi_m \left( \displaystyle \vu \right)} d\displaystyle \vu = 1$. By Theorem \ref{theo:a-MIND}, the distribution with pdf $e^{\bm{\theta}^{*T}\phi_m \left( \displaystyle \vu \right)}$ maximizes (\ref{A-MIND}).

\subsection{Proof of Theorem \ref{theo:fund}}
\label{proof:theo:fund}
Theorem \ref{theo:fund}-\ref{eq:fund_1} is a consequence of the uniqueness of the solution to the Hausdorff moment problem. 

Indeed, the uniform distribution on $[0, 1]$ is uniquely characterized by the sequence of moments $\forall j, E(u^j)=1/(1+j)$ (\cite{shohat1943problem}). Thus, we may replace the uniform marginal constraints in (\ref{MIND}) with the constraints $\forall i, j, ~ E(u^j_i)=1/(1+j)$. The only difference between (\ref{A-MIND}) and (\ref{MIND}) is that the former has $k$ moment constraints whereas the latter has all moment constraints. It follows that $$\forall m > 0,  ~~ h_{\text{AM}} \left(\displaystyle \vu; \phi_m^k,  \bm{\beta}_{m} \right) \underset{k \rightarrow\infty}{\longrightarrow}  h_{\text{M}} \left(\displaystyle \vu; \psi_m,  \bm{\beta}_{m} \right).$$
Additionally, a direct application of the basic consistency theorem for extremum estimators (see \cite{newey1994large}) to (\ref{CVX-MIND}) shows that for any consistent estimator $\bm{\hat{\beta}}_{m,n}$ of $\bm{\beta}_{m}$, $$h_{\text{AM}} \left( \displaystyle \vu; \phi_m^k,  \bm{\hat{\beta}}_{m,n} \right) \underset{n \rightarrow\infty}{\longrightarrow}  h_{\text{AM}} \left( \displaystyle \vu; \phi_m^k,  \bm{\beta}_{m} \right).$$

As for Theorem \ref{theo:fund}-\ref{eq:fund_2}, we simply need to prove that: $$\forall k > 0,  ~~ h_{\text{AM}} \left(\displaystyle \vu; \phi_m^k,  \bm{\beta}_{m} \right) \underset{m \rightarrow\infty}{\longrightarrow}  h\left( \displaystyle \vu_z \right).$$

We recall that $h\left( \displaystyle \vu_z \right) = - KL\left(P_{\displaystyle \vu_z} \vert \vert U\right)$. Using the NWJ characterization of the KL divergence (\cite{nguyen2010estimating}), we get:
\begin{align}
\label{eq:proof:nwj}
h\left( \displaystyle \vu_z \right) = \inf_{T \in \mathcal{T}} ~ -E_{P_{\displaystyle \vu_z}}\left[ T(\displaystyle \vu) \right] + \int_{[0, 1]^d} e^{T(\displaystyle \vu)-1} d\displaystyle \vu,
\end{align}
where $\mathcal{T}$ is the space of continuous functions on $[0, 1]^d$. Using the fact that $\left( \phi_m \right)_m$ is dense in $\mathcal{T}$ (property (P3)), we may rewrite Equation (\ref{eq:proof:nwj}) as 
\begin{align}
h\left( \displaystyle \vu_z \right) =& \lim_{m \to \infty} \min_{\bm{\theta}_m} ~ -1-\bm{\theta}_m^T \underbrace{E_{P_{\displaystyle \vu_z}}\left[ \phi_m(\displaystyle \vu) \right]}_{:= \bm{\alpha}_m} \nonumber \\
&+ \int_{[0, 1]^d} e^{\bm{\theta}_m^T \phi_m(\displaystyle \vu)} d\displaystyle \vu,
\end{align}
where $\bm{\theta}_m$ satisfies $ \bm{\theta}_m^T \phi_m(\displaystyle \vu) = T(\displaystyle \vu)-1$. Note that the inner optimization problem has the same minimizer as (\ref{CVX-MIND}), and the minimum is $h_{\text{AM}} \left( \displaystyle \vu; \phi_m^k,  \bm{\beta}_{m} \right)$, where we have used the fact that the first coordinate of $\phi_m$ is $1$. 

Putting everything together, we get $$\forall k, ~ h\left( \displaystyle \vu_z \right) = \lim_{m \to \infty} h_{\text{AM}} \left( \displaystyle \vu; \phi_m^k,  \bm{\beta}_{m} \right).$$

\end{document}